\documentclass[letterpaper, 10 pt, conference]{ieeeconf}  % Comment this line out if you need a4paper

\IEEEoverridecommandlockouts                              % This command is only needed if 
\usepackage[english]{babel}

\makeatletter
\adddialect\l@ENGLISH\l@english
\makeatother

\usepackage{cite}

\usepackage{epsfig}
\usepackage{pst-all, graphics, graphicx, color}
\usepackage[crop=pdfcrop]{pstool}
\usepackage{psfrag}
\usepackage{subfigure}
\usepackage{amsmath, amssymb} % assumes amsmath package installed
\usepackage{epstopdf}
\usepackage{tikz}
\usepackage{multirow}
\usepackage{float}
\usepackage{algorithm, algorithmic}
\usepackage{dsfont}

\usepackage[makeroom]{cancel}
\floatstyle{ruled}
\newfloat{algorithm}{tbp}{loa}
\providecommand{\algorithmname}{Algorithm}
\floatname{algorithm}{\protect\algorithmname}

\makeatletter
\newcommand\footnoteref[1]{\protected@xdef\@thefnmark{\ref{#1}}\@footnotemark}
\makeatother

\renewcommand{\algorithmiccomment}[1]{\bgroup\hfill\scriptsize//~#1\egroup}

\def\tr{^{\rm T}}

\def\zero{\hbox{\bf 0}}

\def\bff{{\mbox{\boldmath $f$}}}

\def\bfu{{\mbox{\boldmath $u$}}}

\def\bfx{{\mbox{\boldmath $x$}}}

\def\bfI{{\mbox{\boldmath $I$}}}

\def\bfK{{\mbox{\boldmath $K$}}}

\def\bfP{{\mbox{\boldmath $P$}}}

\def\bfX{{\mbox{\boldmath $X$}}}

\def\bfLambda{{\mbox{\boldmath $\Lambda$}}}

\def\bfPhi{{\mbox{\boldmath $\Phi$}}}

\def\bfmu{{\mbox{\boldmath $\mu$}}}

\newtheorem{theorem}{\bf Theorem}

\begin{document}

\title{\LARGE \bf
Incremental Skill Learning of Stable Dynamical Systems}

\author{Matteo Saveriano$^{1}$ and Dongheui Lee$^{1,2}$%
\thanks{$^{1}$Institute of Robotics and Mechatronics, German Aerospace Center (DLR), We{\ss}ling, Germany {\tt matteo.saveriano@dlr.de}.}%
\thanks{$^{2}$Human-Centered Assistive Robotics, Technical University of Munich, Munich, Germany {\tt dhlee@tum.de}.}%
\thanks{This work has been supported by Helmholtz Association.}
}

%\markboth{IEEE Transactions on Robotics,~Vol.~x, No.~x, xxx~2015}%
%{Saveriano \MakeLowercase{\textit{et al.}}: Incremental Reshaping of Stable Dynamical Systems using Gaussian Process Regression}

\maketitle
%\thispagestyle{empty}
%\pagestyle{empty}

%%%%%%%%%%%%%%%%%%%%%%%%%%%%%%%%%%%%%%%%%%%%%%%%%%%%%%%%%%%%%%%%%%%%%%%%%%%%%%%%

\begin{abstract}
Efficient skill acquisition, representation, and on-line adaptation to different scenarios has become of fundamental importance for assistive robotic applications. In the past decade, dynamical systems (DS) have arisen as a flexible and robust tool to represent learned skills and to generate motion trajectories. This work presents a novel approach to incrementally modify the dynamics of a generic autonomous DS when new demonstrations of a task are provided. A control input is learned from demonstrations to modify the trajectory of the system while preserving the stability properties of the reshaped DS. Learning is performed incrementally through Gaussian process regression, increasing the robot's knowledge of the skill every time a new demonstration is provided. The effectiveness of the proposed approach is demonstrated with experiments on a publicly available dataset of complex motions.
\end{abstract}

%\begin{IEEEkeywords}
%Incremental learning ...
%\end{IEEEkeywords}

\IEEEpeerreviewmaketitle

%===================== INTRODUCTION ===========================================%
\section{Introduction}\label{sec:intro}
Future robots will have a tight interaction with humans and they will need an increased versatility to rapidly adapt their behaviour to dynamic and potentially unseen situations. Having a fixed set of predefined skills is not sufficient to execute everyday tasks in human populated environments. Programming by Demonstrations (PbD) is a well-established approach to rapidly teach new skills avoiding tedious programming \cite{Billard_PbD, Calinon18}. In the PbD framework, the robot can learn by observing the human behaviour (imitation learning) \cite{Calinon07, Lee10}, or an expert user can directly guide the robot towards the task execution (kinesthetic teaching) \cite{Lee11, Saveriano15}. %The learned skill is usually represented in a compact way which allows the generalization in different scenarios and reduces memory requirements \cite{Calinon07}.%. Typical examples are Gaussian mixture models/regression (GMM/GMR) \cite{GMR} and hidden Markov models (HMM) \cite{HMM}.

Point-to-point motions, also called discrete movements, are spatial motions ending at a specified target. Discrete movements are of importance in several robotic applications, e.g. in assembly tasks, and they can be combined to build complex tasks \cite{Caccavale18}. Recent work in PbD \cite{DMP, TP-DMP, SEDS, NeuralLearn2,  Clf, tau-SEDS, Perrin16} focuses on representing discrete movements as stable dynamical systems (DS), learned from human demonstrations. DS have been proven to be flexible enough to accurately represent complicated motions \cite{NeuralLearn2,  Clf, tau-SEDS}. Moreover, robots driven by stable DS are guaranteed to reach the desired position, and can react in real-time to external perturbations, like changes in the target position or unexpected obstacles \cite{DS_avoidance, Saveriano13, Saveriano14, Saveriano17, Hoffmann09}. DS have been also used to learn impedance behaviors from demonstrations \cite{Calinon10, SaverianoURAI14} and to refine learned behaviors through reinforcement learning \cite{Kormushev10, Buchli11, Winter15}.%Dynamical systems have been effectively adopted also to learn impedance behaviours from demonstrations \cite{Calinon10, Kronander12, SaverianoURAI14} and to refine learned kinematics or impedance behaviours through reinforcement learning \cite{Kober09, Kormushev10, Buchli11, Kober13, Calinon13, Winter15}.
 %Although DS are widely used in robotics applications, there is not much work on incremental learning approaches for DS. In \cite{Kronander15}, authors propose to reshape the velocity of an autonomous (time independent) first-order DS using a modulation matrix. A similar approach has been effectively used in \cite{DS_avoidance, Saveriano13, Saveriano14} for reactive collision avoidance. The modulation matrix used in \cite{Kronander15} is parametrized as a rotation matrix $\bfR$ multiplied by a scalar gain $\alpha$. The parameters $[\alpha,~\bfR]$ are incrementally learned from demonstrations using Gaussian processes regression \cite{Rasmussen06}. The approach in \cite{Kronander15} has the advantage to locally modify the DS dynamics, i.e. the dynamics in regions of the state space far from the demonstrated trajectories remain unchanged, but it does not preserve global convergence properties of the reshaped system. Moreover, it only works for first-order DS, while second-order DS have been effectively used in robotics to generate dynamically feasible trajectories \cite{DMP, Haddadin10, Calinon12}.
\begin{figure}[t!]
	\centering
	\includegraphics[width=0.85\columnwidth]{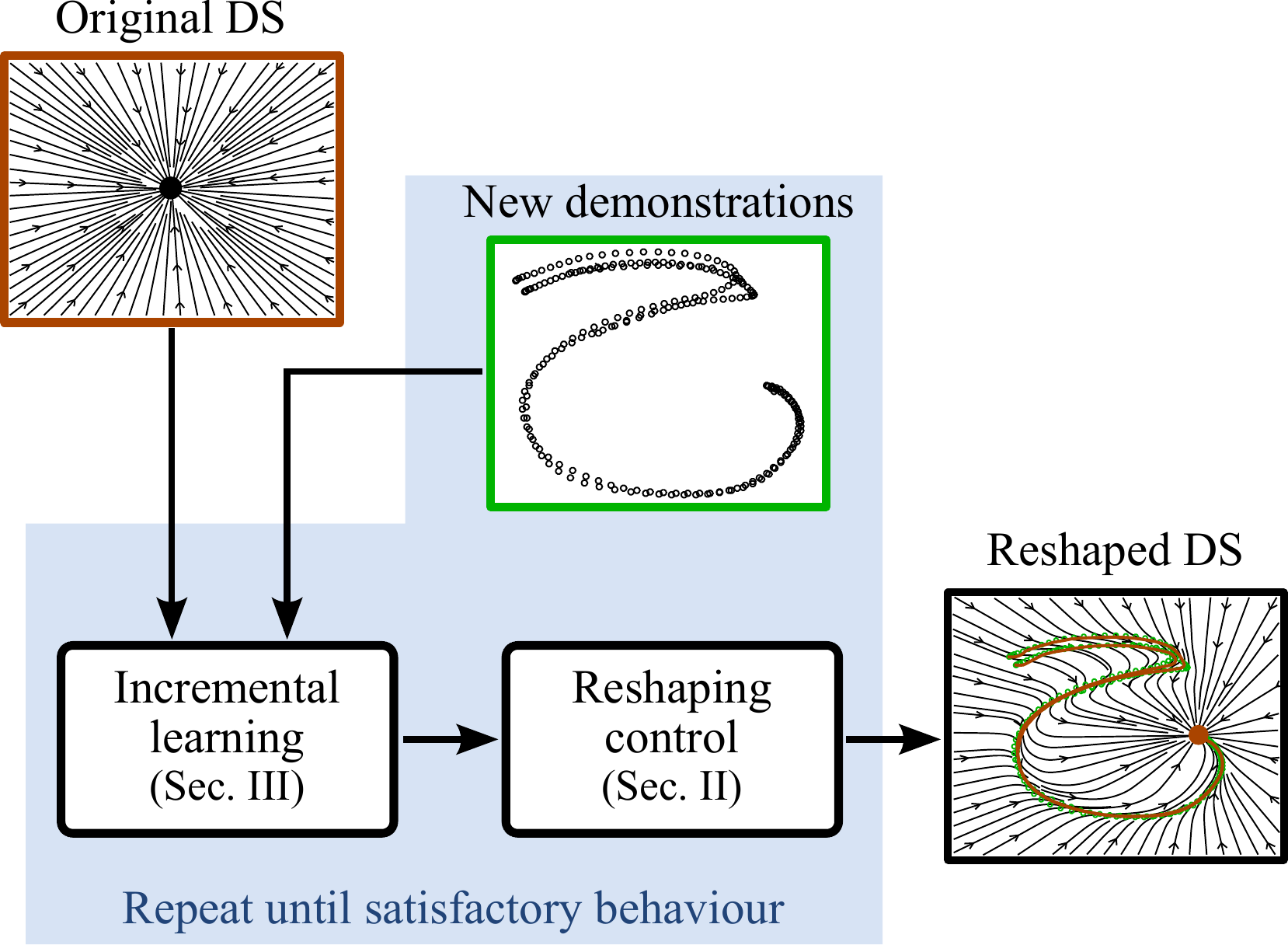}
	\caption{Overview of the Reshaped Dynamical Systems (RDS) approach.}
	\label{fig:overview}
\end{figure} 

This work focuses on the incremental learning of point-to-point motions represented as stable dynamical system, i.e. on how to modify the robot's behavior as novel demonstrations of the task are provided. In our framework, we assume that a predefined skill is given in the form of a stable DS, the so-called original system. The trajectories of the original DS are modified by a reshaping control input, retrieved on-line by means of Gaussian process regression \cite{Rasmussen06}. The reshaping action is learned incrementally from user demonstrations by adding new points to the training set. To alleviate the problem of the increasing computational time, a trajectory-based sparsity criteria is introduced to reduce the amount of novel points added to the training data. The reshaping controller guarantees an accurate reproduction of the demonstrated task without affecting the convergence properties of the original DS. Our formulation does not require any prior knowledge on the original DS and it applies to a wide class of non-linear, autonomous systems.  %In particular, there are no restrictions regarding the order of the DS to reshape and the dimension of the system state space.  
An overview of the proposed Reshaped Dynamical Systems (RDS) is shown in Fig. \ref{fig:overview}. %, where the user observes the robot's behaviour and, if needed, (s)he incrementally provides novel demonstrations until the robot's behaviour is satisfactory. 

%To summarize, our contribution is two-fold:
%\begin{itemize}
%	\item An approach to reshape dynamical systems preserving their convergence properties.
%	\item An approach, based on Gaussian processes, to incrementally learn the reshaping control input.
%\end{itemize}
The rest of the paper is organized as follows. Section \ref{sec:DS_Modulation} presents an approach to modify the trajectory of a DS without affecting its stability. The proposed incremental learning algorithm is described in Section \ref{sec:learning}. Section \ref{sec:rel_work} describes the related works. RDS is evaluated on a public dataset and compared with the state-of-the-art approaches \cite{Clf, Kronander15} in Section \ref{sec:experiments}. Section \ref{sec:conclusion} concludes the paper.

%===================== SECTION II ===========================================%
\section{Dynamical System Reshaping}\label{sec:DS_Modulation}
%In this section, we describe an approach to reshape the dynamics of a generic dynamical system without modifying its stability properties. %For a clear presentation, we first present results for first-order dynamical systems. The extension to high-order DS is discussed in Sec. \ref{sec:ext_high_ds}.

\subsection{Problem definition}\label{subsec:background}
Assume that a robotic skill is encoded in a first-order DS
\begin{equation}
	\dot{\bfx} = \bff(\bfx),
	\label{eq:ds_syst}
\end{equation}
where\footnote{We omit the time dependency of $\bfx$ to simplify the notation.} $\bfx, \dot{\bfx} \in \mathbb{R}^n$ are respectively the position and the velocity of the robot (in joint or Cartesian space), and $\bff : \mathbb{R}^{n} \rightarrow \mathbb{R}^{n}$ is a continuous and continuously differentiable non-linear function. A solution of (\ref{eq:ds_syst}), namely $\bfPhi(\bfx_0,t) \in \mathbb{R}^{n}$, is called trajectory. Different initial conditions $\bfx_0$ generate different trajectories. A point $\hat{\bfx} : \bff(\hat{\bfx}) = \mathbf{0} \in \mathbb{R}^{n}$ is an equilibrium point. An equilibrium $\hat{\bfx} \in S \subset \mathbb{R}^{n}$ is locally asymptotically stable (LAS) if $\lim_{t\rightarrow+\infty} \bfPhi(\bfx_0,t) =\hat{\bfx}, \forall \bfx_0 \in S$. If $S = \mathbb{R}^{n}$, $\hat{\bfx}$ is globally asymptotically stable (GAS) and it is the only equilibrium of the DS.
% A sufficient condition for $\hat{\bfx}$ to be GAS is that there exists a scalar, continuously differentiable function of the state variables $V(\bfx) \in \mathbb{R}$ satisfying (\ref{eq:lyap_stab_cond1})-(\ref{eq:lyap_stab_cond3}) according to \cite{Slotine91}
%\begin{subequations}
%	\label{eq:lyap}
%	\begin{align}
%		&V(\bfx) \geq 0, ~\forall \bfx \in \mathbb{R}^{m n} ~\text{and}~ V(\bfx) = 0 \Longleftrightarrow \bfx = \hat{\bfx} \label{eq:lyap_stab_cond1}\\
%		&\dot{V}(\bfx) \leq 0, ~\forall \bfx \in \mathbb{R}^{m n} ~\text{and}~  \dot{V}(\bfx) = 0 \Longleftrightarrow \bfx = \hat{\bfx} \label{eq:lyap_stab_cond2}\\
%		&V(\bfx) \rightarrow \infty ~\text{as}~ \Vert \bfx \Vert  \rightarrow \infty ~\text{(radially unbounded)} \label{eq:lyap_stab_cond3}
%	\end{align}	
%\end{subequations} 
%Note that, if condition (\ref{eq:lyap_stab_cond3}) is not satisfied, the equilibrium point is LAS. A function satisfying (\ref{eq:lyap_stab_cond1})-(\ref{eq:lyap_stab_cond2}) is called a Lyapunov function\footnote{The interested reader is referred to \cite{Slotine91} for an exhaustive treatment of non-linear systems dynamics.}.

Given a novel demonstration of the task, our goal is to learn a control input  %$\bfu(\bfx) \in \mathbb{R}^{mn}$ and a reshaped DS
%\begin{equation}\label{eq:resh_discr_ds}
%	\dot{\bfx} = \bff(\bfx) + \bfu(\bfx)
%\end{equation}
that satisfies the following requirements:
\begin{enumerate}
	\item The reshaped DS has the same GAS equilibrium of $\dot{\bfx} = \bff(\bfx)$.
	\item The trajectories of the reshaped DS follow the given demonstrations $\mathcal{D}$.
	\item The control input is incrementally updated as novel demonstrations are given.
\end{enumerate}
A control structure that satisfies requirement 1) is presented in Sec. \ref{sec:glob_control}. An approach to learn a control input that satisfies 2) and 3) is presented in Sec. \ref{sec:learning}.

\subsection{Globally stable reshaping controller}\label{sec:glob_control}
In order to modify the trajectories of \eqref{eq:ds_syst}, one can exploit a suitable control input $\bfu(\bfx) \in \mathbb{R}^{n}$. Instead of considering the general form of controlled DS $\dot{\bfx} = \bff(\bfx,\bfu)$, we assume an additive and smooth (continuous and continuously differentiable) control input, obtaining the controlled DS $\dot{\bfx} = \bff(\bfx) + \bfu(\bfx)$. In general, the controlled system is not guaranteed to converge to the same equilibrium point $\hat{\bfx}$ of \eqref{eq:ds_syst}. Hence, we exploit a clock signal to suppress the control input $\bfu(\bfx)$ after $t_f$ seconds, ensuring global convergence to $\hat{\bfx}$. The resulting reshaped DS can be written as
%\begin{equation}
%\left\lbrace
%	\begin{split}
%	 \dot{\bfx} &= \bff(\bfx) + s\bfu(\bfx) \\
%	 \dot{s} &= \alpha(\hat{s}-s) 
%	\end{split}
%\right . ,
%	\label{eq:resh_ds_2}
%\end{equation}
\begin{subequations}
\begin{align}
 \dot{\bfx} &= \bff(\bfx) + s\bfu(\bfx) \label{eq:resh_ds_2}\\
 \dot{s} &= \alpha(\hat{s}-s) \label{eq:clock}
\end{align}
\end{subequations}
where $\hat{s} = 1$ for $t \leq t_f$, $\hat{s} = 0$ for $t > t_f$, and the control input $\bfu(\bfx)  \in \mathbb{R}^{n}$ is a continuous and continuously differentiable function. The time $t_f>0$, after which $\hat{s}=0$, is a tunable parameter. The scalar gain $\alpha > 0$ determines how fast $s$ reaches the desired value $\hat{s}$ and it can be tuned considering that, in practice, $s = \hat{s}$ after $5/\alpha$ seconds. In all the experiments we choose $\alpha = 10$ to have $s = \hat{s}$ after $0.5\,$s. %Being $\hat{s} = 1$ for $t < t_f$, the reshaping control input is active for $t_f$ seconds, where . %Further details about the selection of the control time are provided in Sec. \ref{subsec:max_time}.

The reshaped dynamics \eqref{eq:resh_ds_2}--\eqref{eq:clock} is GAS if the original DS (\ref{eq:ds_syst}) is GAS, as stated by the following theorem.
%\begin{theorem}
%\label{th:discr_reshape}
%If the dynamical system $\dot{\bfx} = \bff(\bfx)$ in (\ref{eq:ds_syst}) has a GAS equilibrium $\hat{\bfx} \in \mathbb{R}^{n}$, then the reshaped dynamics (\ref{eq:resh_ds_2}) globally asymptotically converges to $\hat{\bfx}_r = [\hat{\bfx},0]\tr \in \mathbb{R}^{n+1}$.
%\end{theorem}
\begin{theorem}
\label{th:discr_reshape}
Assume that the dynamical system $\dot{\bfx} = \bff(\bfx)$ in \eqref{eq:ds_syst} has a GAS equilibrium $\hat{\bfx} \in \mathbb{R}^{n}$ and that $\hat{s}$ in \eqref{eq:clock} is $\hat{s} = 0$ for $t > t_f$. Under these assumptions the reshaped dynamics \eqref{eq:resh_ds_2} globally asymptotically converges to $\hat{\bfx}$.
\end{theorem}
\begin{proof}
Note that the linear dynamics of $\dot{s}$ in \eqref{eq:clock} does not depend on the dynamics of $\dot{\bfx}$ in \eqref{eq:resh_ds_2}. Being $\alpha > 0$ and $\hat{s} = 0$ for $t > t_f$, we can conclude that $s$ converges to $\hat{s} = 0$ for $t\rightarrow+\infty$. Hence, for $t\rightarrow+\infty$, the term $s\bfu(\bfx)\rightarrow\mathbf{0}$ and $\bfx$ converges to $\hat{\bfx}$, the GAS equilibrium of (\ref{eq:ds_syst}).
\end{proof}

The formulation introduced in \eqref{eq:resh_ds_2}--\eqref{eq:clock} ensures that the robot's motion is generated by a stable (first-order) dynamical system. As discussed in Sec. \ref{sec:intro}, stable DS generate converging motions that accurately reproduce the demonstrations and are robust to external perturbations. An additive control input is assumed in \eqref{eq:resh_ds_2}. While this is a common assumption for many physical systems like robots, it also eases the computation of the training data as detailed in Sec. \ref{subsec:train_data}. The clock signal in \eqref{eq:clock} introduces a time dependency in the reshaped system. This makes easy to ensure stability properties (see Theorem \ref{th:discr_reshape}) without losing some benefits of autonomous DS in case of external perturbations. Indeed, if the robot is blocked and time passes, the control input remains unchanged because it only depends on the robot position. When the robot is released, it smoothly continues its motion towards the goal.

The proposed reshaped DS (\ref{eq:resh_ds_2}) resembles the dynamic movement primitives (DMPs) framework \cite{DMP}. DMPs reshape a second-order linear system (original DS) with a non-linear forcing term (control input). An exponentially decaying clock signal is used to cancel the effects of the forcing term guaranteeing global stability. Compared to the original DMPs, our approach differs in the following aspects. In our framework the original DS can be any non-linear system. As experimentally shown in Sec. \ref{sec:experiments}, the adoption of a non-linear DS significantly improves the accuracy in reproducing complex, intrinsically non-linear movements. {Moreover, the non-linear control input can be learned and incrementally updated from multiple demonstrations, while batch learning from a single demonstration is used in original DMPs. The adoption of multiple demonstrations improves the generalization capabilities of the learning algorithm.} 

\section{Learning the reshaping controller}\label{sec:learning}
In this section, an approach is presented to learn and on-line retrieve the control input $\bfu(\bfx) \in \mathbb{R}^{n}$ in (\ref{eq:resh_ds_2}) for each state $\bfx$. We firstly describe how to compute the training data from the given demonstrations and learn the reshaping controller. Then, an approach is presented to incrementally update the reshaping controller.    

\subsection{Computation of the training data}\label{subsec:train_data}
Consider that a new demonstration of a skill is given as $\mathcal{D} = \{ \bfx_{d}^{t}, ~\dot{\bfx}_{d}^{t}\}_{t=1}^{T_d}$, where $\bfx_{d}^{t} \in \mathbb{R}^{n}$ is the desired state vector (e.g. the robot position) at time $t$, $\dot{\bfx}_{d}^{t}\in \mathbb{R}^n$ is the time derivative of $\bfx_{d}^{t}$ (e.g. the robot velocity), and $T_d$ is the number of samples. To learn the control input $\bfu(\bfx)$ in (\ref{eq:resh_ds_2}) from $\mathcal{D}$, demonstrations are first converted into a set of input/output training data. 

Assuming $s=1$ and considering (\ref{eq:resh_ds_2}), the dynamics of $\dot{\bfx}$ in (\ref{eq:resh_ds_2}) can be re-written as
\begin{equation}
\label{eq:train_ds}
\dot{\bfx} - \bff(\bfx)	= \bfu(\bfx),
\end{equation}
which shows that the desired control input $\bfu(\bfx)$ is a non-linear mapping between $\bfx$ and $\dot{\bfx} - \bff(\bfx)$. Hence, we consider the demonstrated states $\bfX = \{\bfx_{d}^{t}\}_{t=1}^{T_d}$ as input and $\bfLambda = \left\lbrace \dot{\bfx}_{d}^{t} - \bff(\bfx_{d}^{t})\right\rbrace_{t=1}^{T_d}$ as observations (output) of $\bfu(\bfx)$. In other words, the learned controller adds a displacement to $\bff(\bfx)$ which makes the reshaped dynamics close to the demonstrated one (see requirement 2) in Sec. \ref{subsec:background}).

Once the training data are computed, any regression technique can be applied to learn the reshaping controller and retrieve a smooth control input for each value of $\bfx$. In this work, we adopt a local regression technique, namely the Gaussian process (GP) regression \cite{Rasmussen06}. Local regression ensures that $\bfu \rightarrow \mathbf{0}$ when the state is far from the demonstrated trajectories, making possible to locally follow the demonstrations leaving the rest of the trajectory almost unchanged. Note that GP does not require the alignment of input sequences to a common length, being the regression performed considering all the points in the training set.  

\subsection{Gaussian process regression}
%Gaussian processes (GP) \cite{Rasmussen06} are widely used to learn input-output mappings from observations.%\footnote{Due to the limited space, we shortly describe the main aspects of GP and refer to \cite{Rasmussen06} for an exhaustive treatment.}. %GP assume that the data are generated by a set of underlying functions, whose joint probability distribution is Gaussian . 
Gaussian processes (GP) \cite{Rasmussen06} assume that the training input $\bfX$ and output $\bfLambda_i = \{\lambda_i^{t}\}_{t=1}^T$,  where $\lambda_i$ is the $i$-th component of $\bfLambda$, are drawn from the scalar noisy process $\lambda_{i}^{t} = g(\bfx^{t}) + \epsilon \in \mathbb{R}, ~t=1\ldots T_d$. The noise term $\epsilon$ is Gaussian with zero mean and variance $\sigma_{n}^2$, while $g(\cdot)$ is a smooth and unknown function. The joint distribution between training points and the output $\lambda_i^{*}$ at a query point $\bfx^{*}$ is
\begin{equation}
\label{eq:GP_joint_prob}
\begin{bmatrix} \bfLambda_i	 \\ \lambda_i^{*} \end{bmatrix} \sim \mathcal{N}\left(\mathbf{0}, \begin{bmatrix} \bfK_{XX} + \sigma_n^{2}\bfI & \bfK_{x^{*}X}\\ \bfK_{Xx^{*}} & k(\bfx^{*},\bfx^{*}) \end{bmatrix} \right),
\end{equation}
where $\bfK_{x^{*}X} = \{k(\bfx^{*},\bfx_1^{t})\}_{t=1}^{T_d}$, $\bfK_{Xx^{*}} = \bfK_{x^{*}X}\tr$, $k(\cdot,\cdot)$ is a covariance function, and the element $ij$ of the matrix $\bfK_{XX}$ is given by $\{\bfK_{XX}\}_{ij} = k(\bfx_i,\bfx_j)$.

To make predictions, one can consider that the conditional distribution of $\lambda_i^{*}$ given $\bfLambda_i$ can be written as
\small
\begin{equation}
\label{eq:GP_pred}
	\begin{split}
	\lambda_i^{*} | \bfLambda_i  &\sim \mathcal{N}\left(\mu_{\lambda_i^{*} | \bfLambda_i}, \sigma_{\lambda_i^{*} | \bfLambda_i}^2 \right),\\
		\mu_{\lambda_i^{*} | \bfLambda_i} &= \bfK_{x^{*}X} \left( \bfK_{XX} + \sigma_n^{2}\bfI \right)^{-1}\bfLambda_{i},\\
		\sigma_{\lambda_i^{*} | \bfLambda_i}^2 &=k(\bfx^{*},\bfx^{*}) - \bfK_{x^{*}X} \left( \bfK_{XX} + \sigma_n^{2}\bfI \right)^{-1}\bfK_{Xx^{*}}.
	\end{split}
\end{equation}
\normalsize
The mean $\mu_{\lambda_i^{*} | \bfLambda_i}$ is used as an estimate of $\lambda_i^{*}$ given $\bfLambda_i$. Note that the described procedure holds for a scalar output. Hence, $n$ GPs are used to represent the control input $\bfu(\bfx) \in \mathbb{R}^n$. 

In our approach, $k(\cdot,\cdot)$ is the squared exponential function 
\begin{equation}
\label{eq:se_rbf}
k(\bfx_i,\bfx_j) = \sigma_k^2 \exp{\left(-\frac{\Vert \bfx_i - \bfx_j \Vert^2}{2 l}\right)} + \sigma_n^2 \delta(\bfx_i,\bfx_j),
\end{equation}
where $\delta(\bfx_i,\bfx_j) = 1$ if $\Vert \bfx_i - \bfx_j \Vert = 0$ and $\delta(\bfx_i,\bfx_j) = 0$ otherwise. The variance $\sigma_k^2$, the length scale $l$, and the noise sensitivity $\sigma_n^2 $ are positive hyper-parameters which can be hand-tuned or learned from demonstrations \cite{Rasmussen06}. The kernel function (\ref{eq:se_rbf}) guarantees that the control input goes to zero ($\bfu(\bfx) \rightarrow \mathbf{0}$) for points far from the demonstrated positions.

\subsection{Incremental gaussian process updating}\label{sec:incr_gp}
GP regression is computed considering all the training data. If, as in this work, GP hyper-parameters are fixed, incremental GP learning is performed by simply adding new points to the training set. To reduce the computational effort due to the matrix inversion in (\ref{eq:GP_pred}), incremental GP algorithms introduce criteria to sparsely represent incoming data \cite{Kronander15, Csato02}. Following this idea and assuming that $T_d$ data $\{\bfx_{d}^{t},\dot{\bfx}_{d}^{t}-\bff(\bfx_{d}^{t})\}_{t=1}^{T_d}$ are already in the training set, we propose to add a new data point $[\bfx_{d}^{T_d+1},\dot{\bfx}_{d}^{T_d+1}-\bff(\bfx_{d}^{T_d+1})]$ if the cost is defined as  
\begin{equation}
\label{eq:cost}
C^{T_d+1} = \Vert \dot{\bfx}_{d}^{T_d+1}-\bff(\bfx_{d}^{T_d+1}) - \hat{\bfu}(\bfx_{d}^{T_d+1}) \Vert > \bar{c},
\end{equation}
where $\hat{\bfu}$ is the control input predicted at $\bfx_{d}^{T_d+1}$ using only data $\{\bfx_{d}^{t},\dot{\bfx}_{d}^{t}-\bff(\bfx_{d}^{t})\}_{t=1}^{T_d}$ already in the training set. The tunable parameter $\bar{c}$ represents the error in approximating the demonstrated state derivative and it can be easily tuned. For example, if $\dot{\bfx}_{d}$ is the robot velocity, $\bar{c} = 0.2$ means that velocity errors smaller than $0.2\,$m/s are acceptable. %The proposed incremental reshaping approach is summarized in Tab. \ref{tab:incr_learn}.

%\begin{table}[h]
%    \centering
%    \caption{Proposed reshaping approach.}
%    \label{tab:incr_learn}
%    {\renewcommand\arraystretch{1.3} 
%\begin{tabular}{ |c|l| }
%\hline
%\multirow{4}{*}{Batch}  &  Create a set of predefined tasks encoded as stable DS. \\
% & Stable DS can be designed by the user or learned\\
%  &  from demonstrations as in \cite{SEDS, DMP}.  \\ \cline{2-2}
%  & Provide a Lyapunov function $V(\bfx)$ for each DS.\\
%\hline \hline
%\multirow{7}{*}{Incremental}  &  Observe the robot's behaviour in novel scenarios.\\ \cline{2-2}
%  & If needed, provide a corrective demonstration, for\\
%  & example by kinesthetic teaching the robot.\\ \cline{2-2}  
%  & Learn the parameters of the control input (\ref{eq:resh_discr_ds}), as\\
%  & described in Sec. \ref{sec:learning}. Tuning parameters can be setted \\
%  & empirically by simulating the reshaped DS.\\ \cline{2-2}
%  & Repeat until the refined behaviour is satisfactory.\\
%\hline
%\end{tabular}
%    }
%\end{table}

%===================== RELATED WORK ===========================================%
\section{Related Work}\label{sec:rel_work}
\subsection{Skills representation using dynamical systems}\label{sec:batch_ds_learn}
The dynamic movement primitive (DMP) framework \cite{DMP} is one of the first examples of robotic skills representation via DS. DMPs exploit a non-linear forcing term, learned from a single demonstration, to reshape a linear dynamics, and a clock signal to suppress the non-linear force after a certain time guaranteeing the convergence towards the target. Task-parameterized motion primitives \cite{TP-DMP, Pervez18} extend the standard DMP by introducing extra task-dependent parameters useful to adapt robot movements to novel scenarios. %TP-DMP consists of a weighted summation of second order DS, learned from demonstrations using Gaussian mixture models (GMM) \cite{GMR}. 

The stable estimator of dynamical systems (SEDS) in \cite{SEDS} generates stable motions from a non-linear DS, represented by GMM. Global stability is ensured by constraining the GMM parameters to satisfy a set of stability constraints derived from a quadratic Lyapunov function. The main advantage of SEDS is that the learned system is globally stable. The main limitation is that contradictions may occur between the demonstrations and the quadratic stability constraints, preventing an accurate learning of the desired motion.

The accuracy problem is explicitly considered in several works \cite{NeuralLearn2, Clf, tau-SEDS, Perrin16}, showing that complex motions can be accurately represented by non-linear DS. In \cite{NeuralLearn2, Clf} two different approaches are proposed to learn a Lyapunov function which minimizes the contradictions between the stability constraints and the training data, favoring an accurate reproduction of complex motions. \cite{tau-SEDS} learns a diffeomorphic transformation that projects the training data into a space where they are well represented by a quadratic Lyapunov function. {Perrin et al. \cite{Perrin16} propose a fast algorithm to learn diffeomorphic transformations from a single demonstration. % The approach considers only one demonstration and assumes a linear DS.}

RDS is complementary to the state-of-the-art approaches for skill representation via DS. RDS, in fact, incrementally reshapes the trajectories of a given DS (original DS) without affecting its stability properties. The original DS can be either designed by an expert or learned from demonstrations using one of the the aforementioned approaches. 

%The main outcome of the work in \cite{NeuralLearn2, Clf, tau-SEDS, Cont-DS} is that complex motions can be accurately represented by non-linear DS.

%A common idea between [11] and [12] is to learn from demonstrations a Lyapunov function which minimizes the contradictions between the stability constraints and the training data. In [11], the DS is represented by an extreme learning machine (ELM) [13]. ELM parameters are learned by solving a constrained op-
%timization problem, where stability constraints are derived from the learned Lyapunov function. In SEDSII the learned Lyapunov function is used as control Lyapunov function [14] to compute, at run time, a stabilizing control input. Compared to SEDS, both SEDSII and the approach in [11] exhibit a superior reproduction accuracy. The main advantage of SEDSII is its applicability to any learning approach used to represent demonstrations as a dynamical system.

\subsection{Incremental learning of robotic skills}
Several approaches have been proposed to extend the DMP framework to incremental learning scenarios \cite{Schaal98, Gams09, petric14, Nemec_15, Kulvicius13, Maeda17}. In \cite{Schaal98} the recursive least square and a forgetting factor are used to incrementally update the DMP weights. Gams et al. \cite{Gams09} present a two-layered system for incremental learning of periodic movements. The first layer of the system is a DS which extracts the fundamental frequency of the demonstrations. The second layer is a periodic DMP which learns the waveform of the demonstrated motion. The overall system works on-line, but it is limited to periodic motions, while discrete movements are the focus of our work. The work in \cite{petric14} considers incremental human coaching for DMP. In the teaching phase, the user is considered as an obstacle, avoided by adding an extra forcing term to the DMP \cite{Hoffmann09}. In this way, the human is able to modify on-line the robot's path without touching it. The novel path is used to incrementally updated DMP weights via recursive least square. Nemec et al. \cite{Nemec_15} leverage iterative learning control \cite{ILC} to realize a learning strategy which is faster and more robust than recursive least square. The approaches in \cite{petric14, Nemec_15} are evaluated on periodic movements, but they are also applicable to point-to-point motions. The interaction between two agents, modeled via DMPs, is incrementally learned in \cite{Kulvicius13} to guarantee that both agents equilibrate into a common target, i.e. the two agents are effectively helping each other. Maeda et al. \cite{Maeda17} propose active incremental learning with DMP and Gaussian Processes (GP). They learn a GP from demonstrations and use GP regression to retrieve a confidence execution bound. If the confidence bound is low, the robot explicitly asks for novel demonstrations and updates the GP weights. A DMP is then trained over the GP mean to generate a converging trajectory. 

Similarly to DMP, RDS exploits an additive control input and a clock signal to reshape an asymptotically stable DS. The role of the clock signal in RDS and DMP is the same: suppress the control input to guarantee asymptotic stability. In DMP, the control input (or forcing term) is a function of time, while in RDS it is a function of the robot's position. A position dependent control generates smooth motions in case the robot is kept fixed by an external perturbation (see Sec. \ref{sec:glob_control}). In the same situation, a time depended forcing term may generate big accelerations when the robot is released due to the time passed. DMP reshapes only linear spring-damper DS, while the proposed RDS applies to any autonomous DS. This is the main limitation of DMP-based incremental approaches. Indeed, linear DS generate straight trajectories and, as experimentally demonstrated in Sec. \ref{subsec:exp_1}, transforming a straight line into a non-linear path is not trivial and may generate a loss of accuracy, i.e. the learned motion does not accurately match the demonstrated one.

The work in \cite{Blocher17} leverages Contraction theory to automatically compute a stabilizing control input for a DS represented by GMM. Even if the control input can be computed on-line, \cite{Blocher17} only works for DS represented by GMM, while RDS applies to any parameterization. The Locally Modulated Dynamical Systems (LMDS) in \cite{Kronander15} reshapes an autonomous DS using a modulation matrix, obtained by multiplying a rotation matrix by a scalar gain. The modulation matrix is incrementally learned from demonstrations using GP \cite{Rasmussen06}. The learned modulation matrix does not generate any spurious attractor in the modulated DS. Moreover, the effects of the modulation disappear for points far from the demonstrations. These properties of the modulation matrix guarantee the local stability of the modulated DS. Even if global stability is not proved, experiments show that the modulated DS remains stable in practice. % The approach in \cite{Kronander15} has the advantage to locally modify the DS dynamics, i.e. the dynamics in regions of the state space far from the demonstrated trajectories remain unchanged. As a limitation, one can consider that it is hard to generalize \cite{Kronander15} to a generic $n$-dimensional space. Indeed, parameterize a rotation matrix in a space with $n>3$ is not trivial and requires at least $n(n-1)/2$ parameters \cite{mortari2001rigid}, while RDS linearly scales with the space dimension. 
LMDS shares some similarities with the proposed RDS. Like RDS, LMDS applies to any autonomous DS (both linear and non-linear), it allows incremental learning from multiple demonstrations, and it permits an accurate reproduction of demonstrated trajectories. These similar features make interesting to experimentally compare the performance of RDS and LMDS (see Sec. \ref{sec:experiments}).

%===================== EXPERIMENTS ===========================================%
\section{Results and comparisons}\label{sec:experiments}
%\subsection{LASA HandWritten Dataset}\label{sec:lasa}
\begin{figure*}[t]
	\centering
	\subfigure[Linear DS + RDS]{\includegraphics[width=0.49\textwidth]{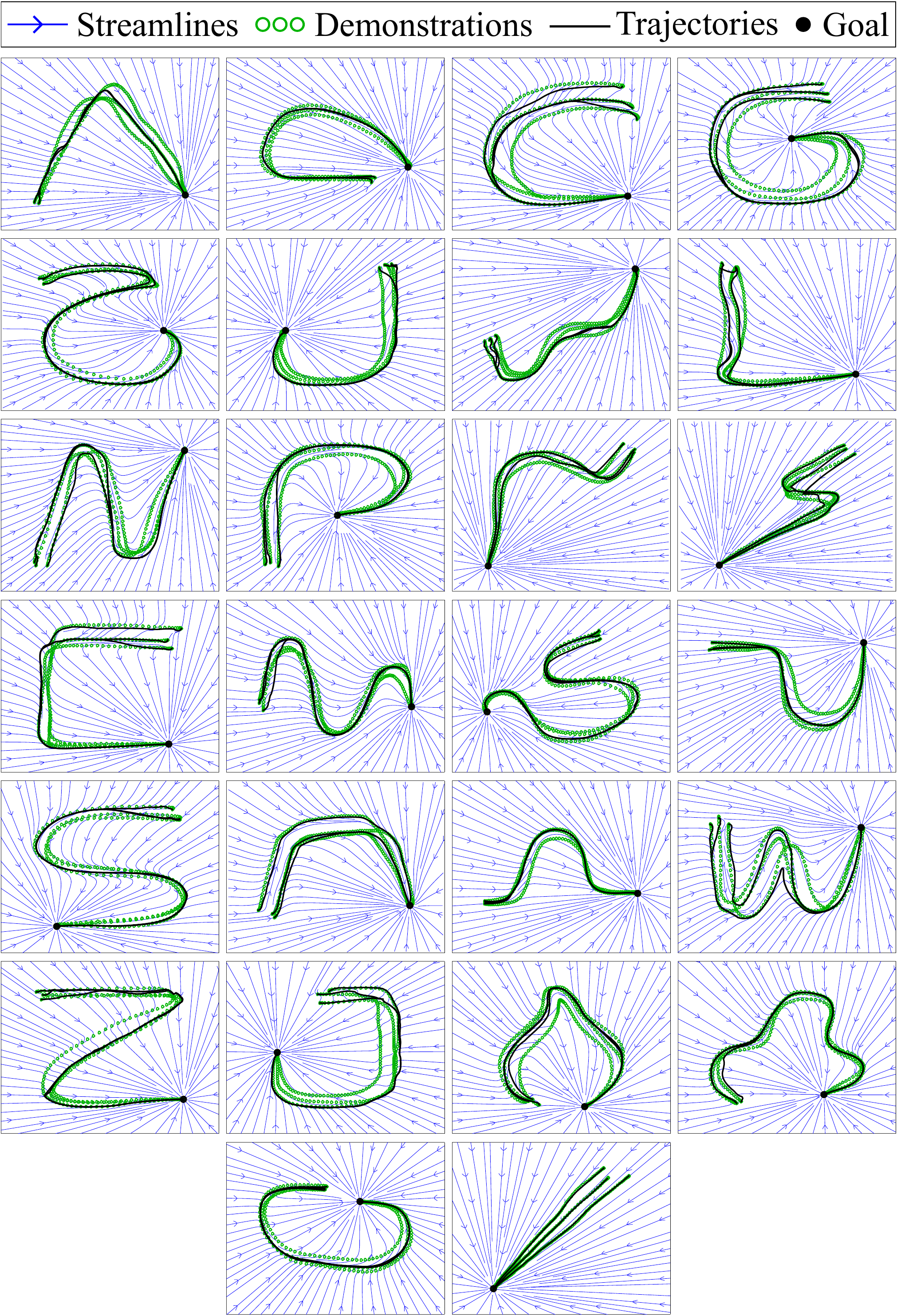}\label{fig:linear_to_lasa}}
	\hspace*{1mm}
	\subfigure[SEDS + RDS]{\includegraphics[width=0.49\textwidth]{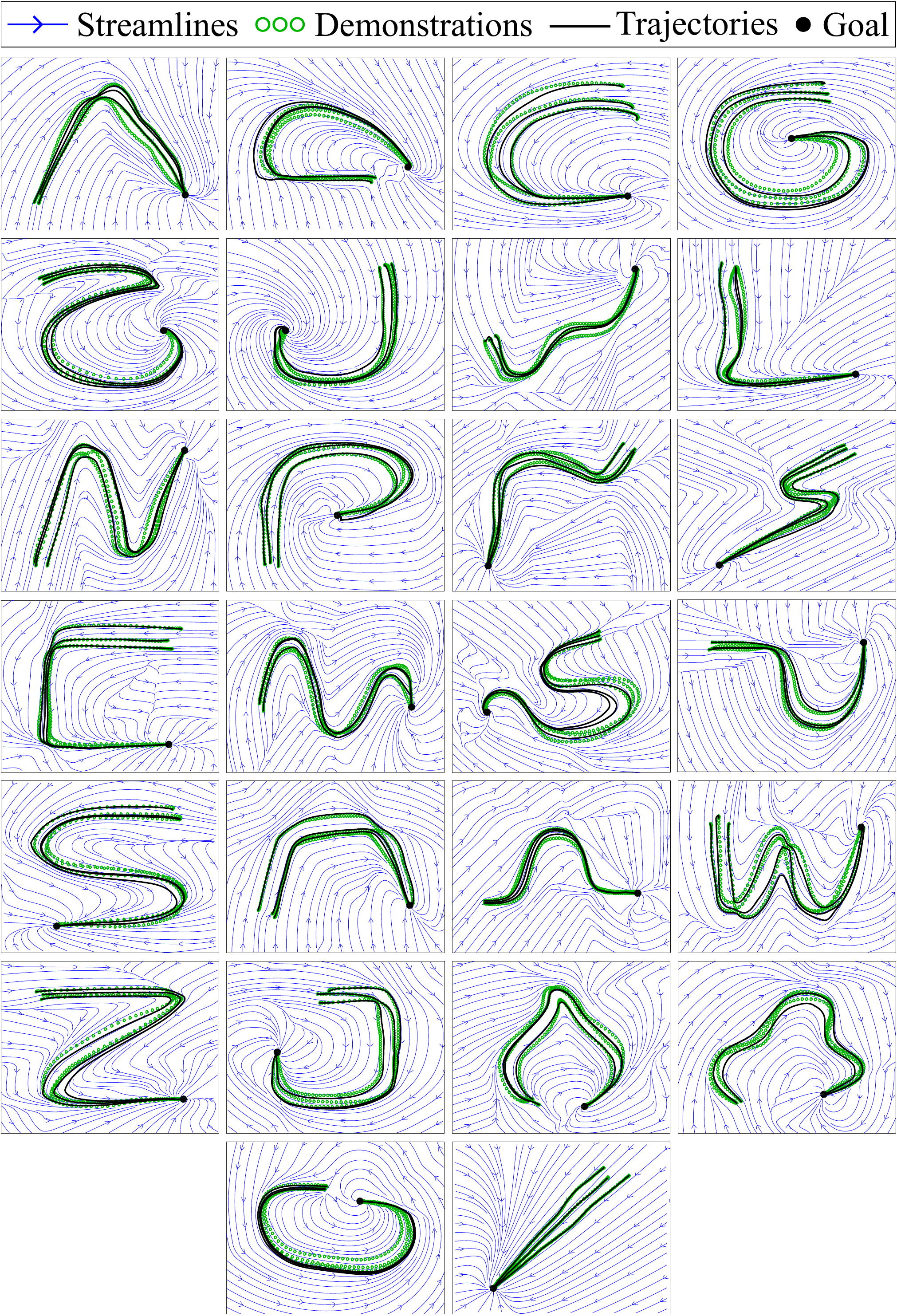}\label{fig:seds_to_lasa}}
    	\caption{The linear DS $\dot{\bfx} = -3 \bfx$ (a) and the SEDS non-linear system (b) reshaped into the $26$ complex motions of the LASA dataset. The proposed RDS is able to learn complex motions without affecting the global stability of the original DS.}
    	\label{fig:lasa_results}
\end{figure*} 
Experiments in this section show the effectiveness of the proposed Reshaped Dynamical Systems (RDS) approach. The LASA Handwriting dataset\footnote{Available on-line: http://bitbucket.org/khansari/lasahandwritingdataset.} is used as a benchmark. The dataset consists of $26$ different point-to-point 2D motions, where each motion trajectory is demonstrated seven times and contains $T_d = 1000$ positions and velocities. All the demonstrations end at the target position $\hat{\bfx} = [0,0]\tr$.

\subsection{Accuracy test}\label{subsec:exp_1}
This experiment compares the reproduction accuracy of the proposed RDS approach and the LMDS approach in \cite{Kronander15}. As a proof of concepts, we consider the first three demonstrations for each motion in the LASA dataset and we subsample each demonstrated trajectory to $100$ samples. LMDS guarantees local stability if the modulation is not active in a neighborhood of the equilibrium. This property is called locality in \cite{Kronander15}. In order to ensure the locality property, we remove the last $10$ points in each demonstration, creating a neighborhood of the origin without training points. We do the same with our approach for a fair comparison. To guarantee the maximum accuracy for both the approaches, all the $90$ samples of each trajectory are considered without applying the sparsity criteria in Sec. \ref{sec:incr_gp}. Moreover, the optimal hyper-parameters are learned from the given demonstrations by maximizing the marginal log-likelihood \cite{Rasmussen06}.

The error that occurs when reproducing a demonstrated motion is measured by the swept error area (SEA) \cite{Clf}, defined as $SEA = \sum_{t=1}^{T_d-1}\mathcal{A}(\bfx_e^{t}, \bfx_e^{t+1}, \bfx_d^{t}, \bfx_d^{t+1})$, where $\bfx_{e}^{t}$ and $\bfx_{d}^{t}$ are respectively the reproduced and the demonstrated position at time $t$, $T_d$ is the length of the demonstration, and $\mathcal{A}(\cdot)$ is the area of the tetragon formed by $\bfx_e^{t}$, $\bfx_e^{t+1}$, $\bfx_d^{t}$, and $\bfx_d^{t+1}$. The reproduced trajectory is equidistantly re-sampled to contain exactly $T_d$ points. The $SEA$ metric measures how well the DS preserves the shape of the demonstrations. To measure how the DS preserves the kinematics of the demonstrations, we use the velocity error defined in \cite{NeuralLearn2} as
$V_{rmse} = \sqrt{\frac{1}{T_d}\sum_{t=1}^{T_d}\Vert \dot{\bfx}_{d}^t - \bff(\bfx_{d}^t)\Vert^2}$, evaluated on the training data $\{\bfx_{d}^{t},\dot{\bfx}_{d}^t\}_{t=1}^{T_d}$. The $V_{rmse}$ measures the difference between the demonstrated velocities and the velocities generated by the learned DS for each training position. %are the demonstrated velocities, $\dot{\bfx}_{r}$ are the reshaped velocities computed at each demonstrated position $\bfx_{d}^{t}$ and $T$ is the number of demonstrated data.  %Alternatively, one can use a time dependent term to guarantee that $\alpha \longrightarrow 0$, $\bfR \longrightarrow \bfI $ for $t \longrightarrow +\infty$. 

Two different scenarios are considered. In the first scenario, the first-order, linear DS $\dot{\bfx} = -3 \bfx$ is used as an original system. This task is particularly challenging, since a linear dynamics has to be transformed into the complex, non-linear motions of the LASA dataset (see the qualitative results in Fig. \ref{fig:linear_to_lasa}). In the second scenario, a stable non-linear DS for each motion is learned by means of the Stable Estimator of Dynamical Systems (SEDS) approach in \cite{SEDS}. RDS and LMDS are applied to the learned DS to improve the reproduction accuracy of the SEDS algorithm.  

The reproduction errors for the considered scenarios are shown in Tab. \ref{tab:comparison_lasa}. Since the reproduction errors ($SEA$ and $V_{rmse}$) of each motion are not normally distributed, we consider the median $M_e$ instead of the mean. To indicate the maximal and minimal deviation from the typical performance, we provide the location of the $10\%$ ($Q_{10}$) and the $90\%$ ($Q_{90}$) quantiles. As shown in Tab. \ref{tab:comparison_lasa}, SEDS+RDS has median errors of $81\,$mm\textsuperscript{2} ($SEA$) and $2\,$mm/s ($V_{rmse}$), which is significantly more accurate than Linear+RDS ($124\,$mm\textsuperscript{2} for $SEA$ and $5.2\,$mm/s for $V_{rmse}$). Also with LMDS, modulating a SEDS system gives more accurate results than modulating a linear DS. This is an expected result, because it is easier to transform a dynamics which is close to the demonstrated motion, rather than transforming a linear DS into a complex motion. In both cases, RDS exhibits higher accuracy than LMDS, meaning that RDS is more effective in imposing a (potentially) different dynamics to a given DS.

\begin{table}[t]
    \centering
    \caption{Reproduction error of RDS and LMDS on the LASA dataset.} 
    \label{tab:comparison_lasa}
    %\resizebox{\columnwidth}{!}{%
    {\renewcommand\arraystretch{1.3} 
	\begin{tabular}{ |c|c|c| }
	\hline
	Learning   & $SEA$ [mm\textsuperscript{2}] & $V_{rmse}$ [mm/s] \\
	Method & ($M_e\,$/$\,Q_{10}\,$/$\,Q_{90}$)  & ($M_e\,$/$\,Q_{10}\,$/$\,Q_{90}$)\\
	\hline
	\hline
	Linear$\,$+$\,$RDS & \textbf{124} / \textbf{25} / \textbf{604} & \textbf{5.2} / 3.4 / \textbf{9.0} \\
	\hline
	Linear$\,$+$\,$LMDS & 233 / 44 / 842 & 6.6 / \textbf{3.2} / 32.4 \\
	\hline
	\hline
	SEDS &  366 / 73 / 584 & 11.5 / 8.2 / 30.9 \\
	\hline
	SEDS$\,$+$\,$RDS &  \textbf{81} / \textbf{27} / \textbf{251} & \textbf{2.0} / \textbf{1.0} / \textbf{3.2} \\
	\hline
	SEDS$\,$+$\,$LMDS &  195 / 66  / 437 & 6.2 / 2.4 / 13.3  \\
    \hline
\end{tabular}
}%}
\end{table}

\subsection{Incremental learning of multi-model behaviors}
The goal of this experiment is two-fold. First, it shows that RDS can learn multi-model motions, i.e. different behaviours in different regions of the space. Second, the experiment shows that RDS and the batch learning approaches in Sec. \ref{sec:batch_ds_learn}  are complementary. As a proof of concepts, we investigate the combination of RDS with SEDSII  \cite{Clf}.

SEDSII computes a stabilizing control input from a learned control Lyapunov function (CLF). The CLF is parameterized as $CLF = \bfx\tr \bfP^{0} \bfx + \sum_{l=1}^{k} \beta^{k}(\bfx)(\bfx\tr \bfP^{k} (\bfx-\bfmu^k))^2$, where $\beta^{k}$, $\bfP^{k}$, and $\bfmu^{k}$ are learned from demonstrations by solving a constrained optimization problem. SEDSII is very effective in accurately learning complex motions while guaranteeing the convergence towards a unique target, but it is not prone to an incremental implementation. Hence, SEDSII is combined with RDS as
\begin{equation}
\dot{\bfx} = \underbrace{\bff_{orig}(\bfx)}_{DS} + \underbrace{\bfu_{CLF}(\bfx)}_{SEDSII} + \underbrace{s\bfu_{RDS}(\bfx)}_{RDS},
\label{eq:seds2_rds}
\end{equation}
where $\bff_{orig}(\bfx)$ is a possibly unstable DS, $\bfu_{CLF}(\bfx)$ is the control input that stabilizes the original DS, and $s\bfu_{RDS}(\bfx)$ is the reshaping controller defined in \eqref{eq:resh_ds_2}--\eqref{eq:clock}.

The controlled DS in (\ref{eq:seds2_rds}) is tested in an incremental learning scenario to show the benefits of combining SEDSII and RDS. We consider the four multi-model motions of the LASA dataset. Each multi-model motion contains demonstrations of two or three different motions (see Fig. \ref{fig:clf_lrds}). Only the first demonstration of each different motion is considered in this experiment. Demonstrations are sub-sampled to $100$ samples. The original DS $\bff_{orig}(\bfx)$ is a Gaussian process, learned from the given demonstration. The parameters used in this experiment are listed in Tab. \ref{tab:multi_par}.

\begin{table}[b]
    \centering
    \caption{Parameters used in the multi-model behaviors learning experiment.} 
    \label{tab:multi_par}
    \resizebox{\columnwidth}{!}{%
    {\renewcommand\arraystretch{1.3} 
	\begin{tabular}{ |ccc|c|ccccc| }
	\cline{1-9}
	\multicolumn{3}{ |c| }{Original GP} & {SEDSII} & \multicolumn{5}{ c| }{RDS}\\ \cline{1-9}
	$\sigma_{k}^{2}$ & $\sigma_{n}^{2}$ & $l$ & k ($\#$ CLF) & $\sigma_{k}^{2}$ & $\sigma_{n}^{2}$ & $l$ & $\hat{c}$ [m/s] & $t_f$ [s]\\
	\hline
	1.0 & 0.4 & 3.0 & 3 & 1.0 & 0.4 & 3.0 & 0.01 & 10.0\\
	\hline
	\end{tabular}
}}
\end{table}	

Three different tests are performed, as shown in Fig. \ref{fig:clf_lrds}. In the first case both $\bff_{orig}(\bfx)$ and $\bfu_{CLF}(\bfx)$ are learned from the first motion (green circles in Fig. \ref{fig:clf_lrds}), while $\bfu_{RDS}(\bfx)=\mathbf{0}$. Novel demonstrations are then provided (brown and red circles) and $\bff_{orig}(\bfx)$ is incrementally updated as proposed in Sec. \ref{sec:incr_gp}. The CLF is not re-trained, since CLF parameters estimation cannot be performed efficiently. As shown in Fig. \ref{fig:clf_lrds}(a) and Tab. \ref{tab:comparison_multi_modal}, novel demonstrations are poorly represented, especially if different motions are demonstrated with the initial CLF parameters. In the second case, instead, the CLF is re-trained considering all the demonstrations. SEDSII accurately learns the multi-model motions, but the training takes almost $200$ times longer. The third case shows the combination of SEDSII and RDS. Instead of re-training the CLF parameters, which is computationally expensive, the reshaping term $\bfu_{RDS}(\bfx)$ is incrementally learned. With the parameters in Tab. \ref{tab:multi_par}, the incremental learning approach in Sec. \ref{sec:incr_gp} uses from $45\%$ to $65\%$ of the points to encode the motion. Results in Fig. \ref{fig:clf_lrds}(c) and Tab. \ref{tab:comparison_multi_modal} clearly show that the combination of SEDSII and RDS is a good compromise in terms of accuracy and training time.

\begin{figure}[t]
	\centering
	{%
	\setlength{\fboxsep}{0pt}%
    \subfigure[SEDSII with initial CLF.]{\centering {\includegraphics[width=\columnwidth]{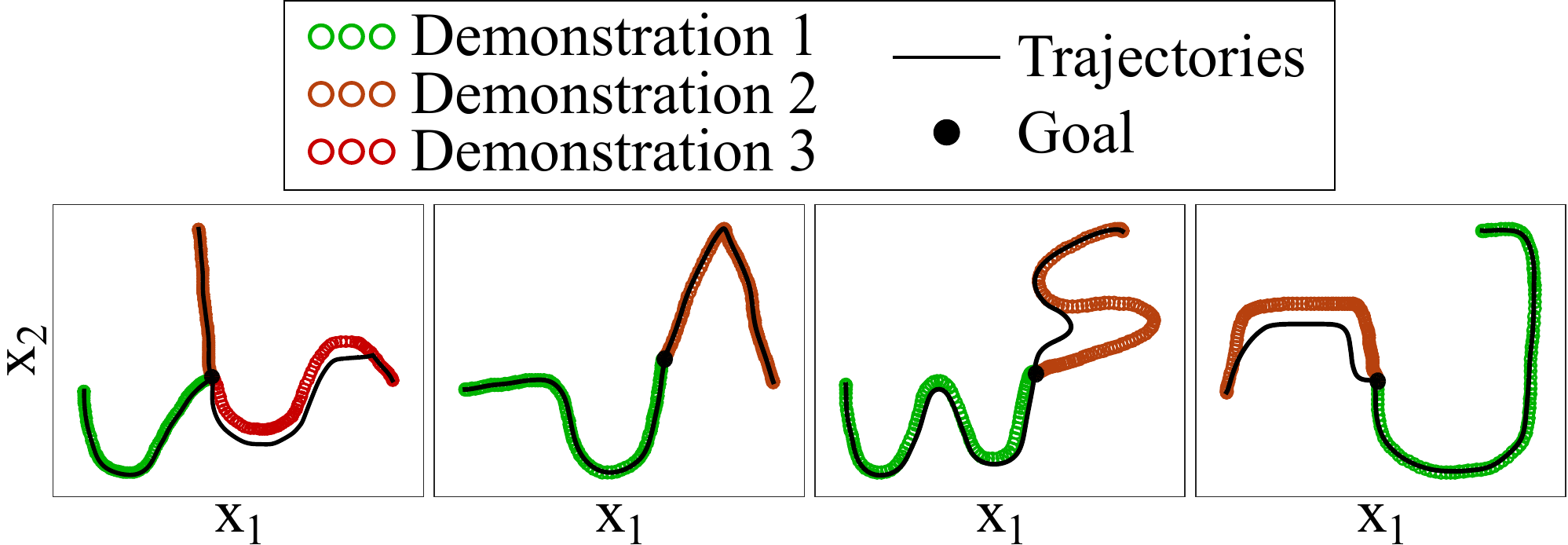}}}
    \subfigure[SEDSII with re-trained CLF.]{\centering {\includegraphics[width=\columnwidth]{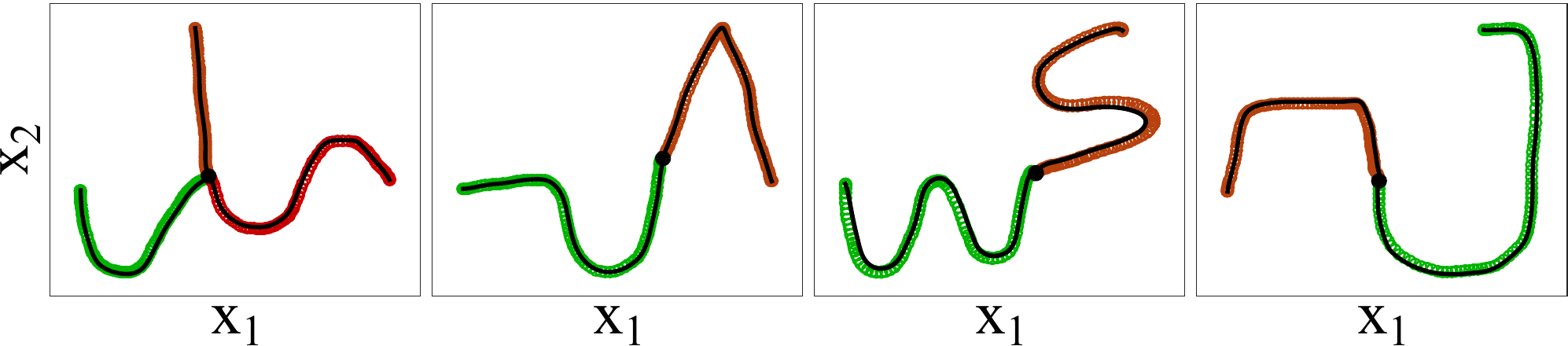}}}
    \subfigure[SEDSII (initial CLF) combined with our reshaping approach.]{\centering {\includegraphics[width=\columnwidth]{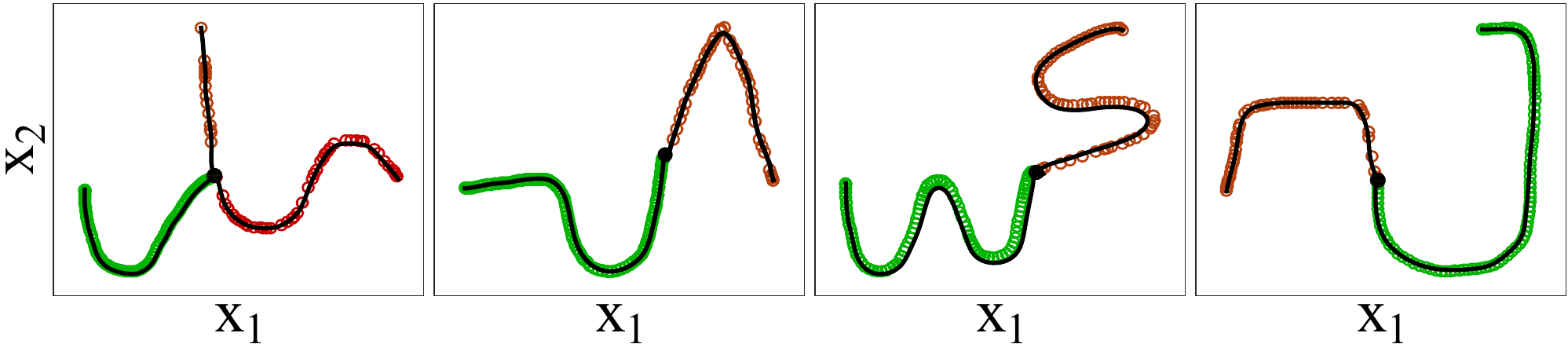}}}
    }%
    	\caption{Qualitative results for the incremental learning of stable multi-model motions with different approaches.}
    	\label{fig:clf_lrds}
\end{figure} 

\begin{table}[t]
    \centering
    \caption{Reproduction errors and Training Times of RDS and SEDSII on the multi-model LASA motions.} 
    \label{tab:comparison_multi_modal}
    \resizebox{\columnwidth}{!}{%
    {\renewcommand\arraystretch{1.3} 
	\begin{tabular}{ |c|c|c|c| }
	\hline
	Learning  & Re-train  & SEA [mm\textsuperscript{2}]  & Time [s]\\
	Method & CLF & ($M_e\,$/$\,Q_{10}\,$/$\,Q_{90}$)  & (mean $\pm$ std) \\
	\hline
	\hline
	GPR$\,$+$\,$SEDSII & No & 1.69 / 1.05 / 5.27 & \textbf{0.009} $\pm$ \textbf{0.007}\\
	\hline
	\hline
	GPR$\,$+$\,$SEDSII & Yes & \textbf{1.54} / 1.08 / \textbf{2.5} & 3.6 $\pm$ 2.6\\
	\hline
	\hline
	GPR$\,$+$\,$SEDSII + RDS & No & 1.56 / \textbf{1.05} / 2.71 & 0.018 $\pm$ 0.013\\
    \hline
\end{tabular}
}}
\end{table}	

%\begin{table}[h]
%    \centering
%    \caption{Reproduction errors and Training Times of RDS and SEDSII on the multi-modal LASA motions.} 
%    \label{tab:comparison_multi_modal}
%    \resizebox{\columnwidth}{!}{%
%    {\renewcommand\arraystretch{1.3} 
%	\begin{tabular}{ |c|c|c|c| }
%	\hline
%	Learning  & Re-trained  & Reproduction Error  & Training Time \\
%	Method & CLF & (mean $\pm$ std) [mm/s] & (mean $\pm$ std) [s]\\
%	\hline
%	\hline
%	GPR + SEDSII & No & 2.43 $\pm$ 1.75 & \textbf{0.002} $\pm$ \textbf{0.0007}\\
%	\hline
%	\hline
%	GPR + SEDSII & Yes &  \textbf{1.64} $\pm$ \textbf{0.57}  & 3.6 $\pm$ 2.6\\
%	\hline
%	\hline
%	GPR + SEDSII + RDS & No &  1.7 $\pm$ 0.65  & 0.018 $\pm$ 0.013\\
%    \hline
%\end{tabular}
%}}
%\end{table}	

\subsection{Incremental learning in higher dimensions}\label{subsec:exp_3}
RDS is directly applicable to spaces of any dimension. On the contrary, LMDS exploits a rotation matrix, and defining a rotation in spaces with more than $3$ dimensions is still an open problem (see Sec. \ref{sec:discussion}). Extending LMDS to high dimensional spaces is beyond the scope of this paper. Hence, in this experiment, we show the scalability of RDS to high dimensional spaces. To this end, we exploit the DS  $\dot{\bfx} = 3(\hat{\bfx}- \bfx)$ to generate a converging trajectory in a $6$ dimensional space. The $6$D state vector $\bfx = [\theta_1,\ldots,\theta_6]\tr \in \mathbb{R}^6$ can be interpreted as the joint angles of a robotic manipulator. The original DS generates a point-to-point motion in the joint space from $\bfx(0) = [35, 55, 15, -65, -15, 50]\tr\,$deg to $\hat{\bfx} = [-60, 30, 30, -70, 25, 85]\tr\,$deg. The original joint angles trajectories are shown in Fig. \ref{fig:ds_7D} (black solid lines).

The original trajectory is modified by providing $100$ additional data in the time interval $[0.25, 1.25]\,$s (brown solid line in Fig. \ref{fig:ds_7D}). For each joint angle $\theta_i$, the training data belongs to a straight line starting from $\theta_i(0.25)\,$deg and ending at $\theta_i(0.25) + 20\,$deg. As shown in Fig. \ref{fig:ds_7D}, the reshaped trajectories (blue solid lines) accurately follow the given demonstration and converge to the desired goal $\hat{\bfx}$. Results are obtained with noise variance $\sigma_{n}^{2} = 0.04$, signal variance $\sigma_{k}^{2} = 0.1$, length scale $l=0.01$, and the threshold $\bar{c} = 1\,$rad/s. With the adopted $\bar{c}$ only $26$ points over $100$ are used to learn the control input.

\begin{figure}[!t]
\centering  
\includegraphics[width=\columnwidth]{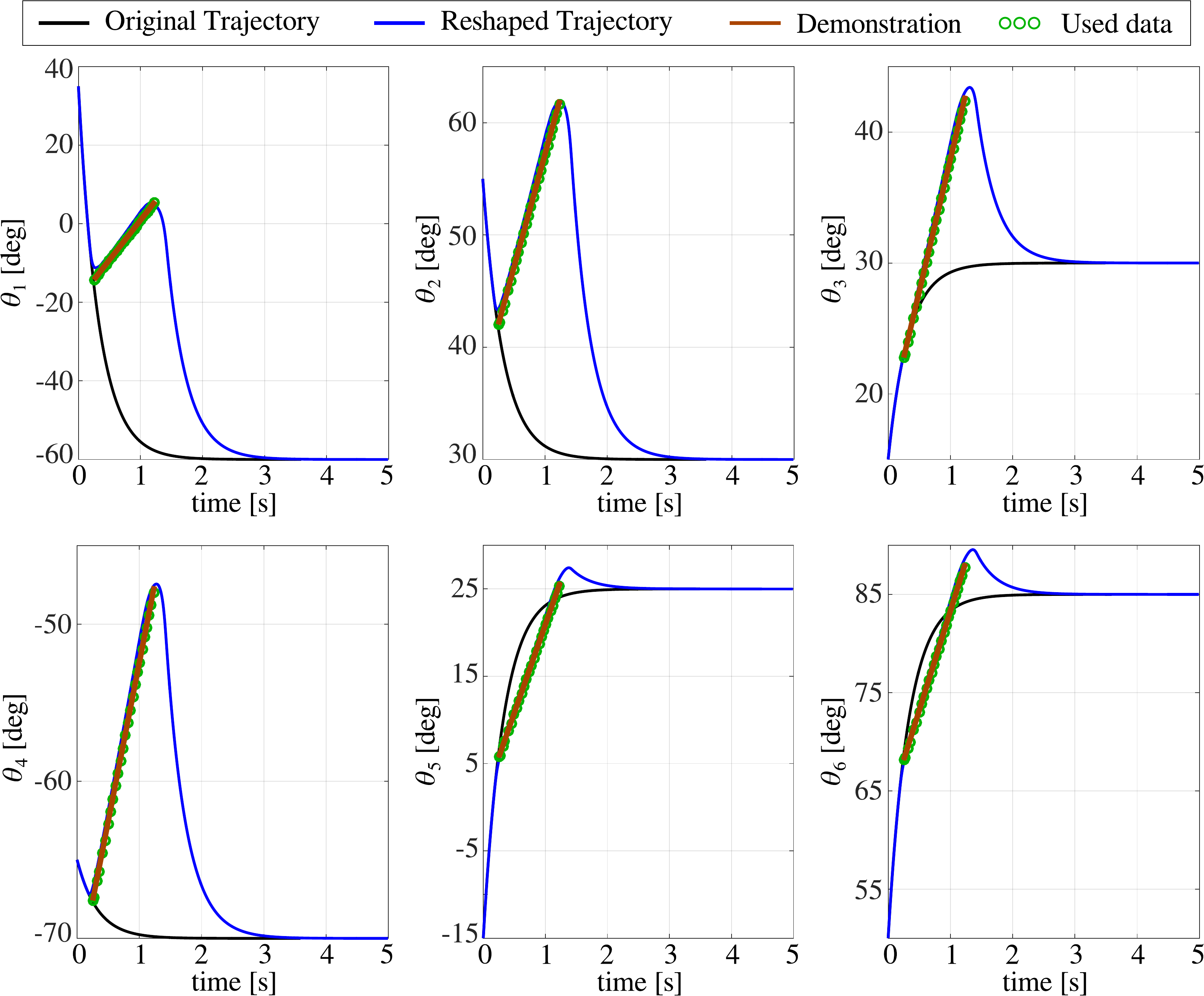}
\caption{Results obtained when RDS reshapes a motion in a $6$D space.}
\label{fig:ds_7D}
\end{figure}

\subsection{Discussion}\label{sec:discussion}
Performed experiments have underlined several properties of the proposed RDS approach. As shown in Sec. \ref{sec:experiments}-B, RDS can be easily combined with batch learning based approaches like SEDSII. The result of this combination is a system capable to incrementally refine a learned skill by significantly reducing training time while preserving the stability of the motion and the reproduction accuracy. Results in Sec. \ref{subsec:exp_1} show that reshaping a non-linear system results in more accurate reproduction than reshaping a linear DS. This is because it is hard to transform linear dynamics (straight lines) into a complex, intrinsically non-linear motion like the ones contained in the LASA dataset (see Fig. \ref{fig:lasa_results}). It is worth noticing that DMPs also reshape a linear dynamical system and they may suffer from a similar accuracy problem. 

RDS, as DMP, exploits a clock signal to suppress the control input after $t_f$ seconds and to guarantee global stability. The value of $t_f$ affects the obtained results. Small values of $t_f$ may result in the loss of accuracy, if the control input is suppressed too early. On the contrary, too large values of $t_f$ may cause the system to stop in a local equilibrium for long time before the control is deactivated.  These problems were not encountered in our experiments. The reason is that we selected large values of $t_f$ (larger than the demonstration time) and, since we used local demonstrations (i.e. no training data were added in a neighborhood of the equilibrium) and a local regression technique (GP), the control input was already vanishing ($\bfu \rightarrow \zero$) for $t < t_f$. In other words, the reshaped DS was reaching the goal before $t_f$ seconds. 

In order to illustrate the behavior of RDS, we design a ``failure'' case where the reshaped system falls into a spurious attractor. Consider that RDS generates a spurious equilibrium if and only if, for $s=1$, $\bfu(\bfx) = -\bff(\bfx)$ for some $\bfx \neq \hat{\bfx}$, i.e. if the original dynamics and the learned control are anti-parallel and have the same magnitude. To reproduce this situation, we reshape the 2D DS $\dot{\bfx} = -3\bfx$, used to generate a converging motion from $\bfx(0) = [2,\,2]\tr\,$m to $\hat{\bfx} = [0,\,0]\tr\,$m. As shown in Fig. \ref{fig:spurious_test}, a demonstration is provided in the form of a straight line starting from the goal ($\hat{\bfx} = [0,\,0]\tr\,$m) and ending at $\bfx = [0.6,\,0.6]\tr\,$m, therefore pushing away the original DS from its global equilibrium. In this case, RDS generates a spurious attractor at about $\bar{\bfx} = [0.6,\,0.6]\tr\,$m (Fig. \ref{fig:spurious_test} (left)) because $\bfu(\bar{\bfx}) = -\bff(\bar{\bfx})$. Satisfying the condition $\bfu(\bfx) = -\bff(\bfx)$ for $\bfx \neq \hat{\bfx}$ is improbable in realistic cases, as experimentally shown in this section. For instance, in Sec. \ref{subsec:exp_3} we also reshape $\dot{\bfx} = -3\bfx$ (in a 6D space) without generating spurious equilibria (Fig. \ref{fig:ds_7D}). However, even if the motion temporary stops in a spurious equilibrium, the control input starts to vanish ($s \rightarrow 0$) for $t > t_f$ and the motion converges to the global attractor (Fig. \ref{fig:spurious_test} (right)). 
Depending on the application, waiting $t_f$ seconds in a spurious attractor may be undesirable. Spurious attractors in first-order DS can be detected by checking if the DS velocity vanishes for any $\bfx \neq \hat{\bfx}$ and escaped by adding a small velocity in a fixed direction \cite{DS_avoidance}, e.g. a velocity pointing towards the goal. Moreover, the influence of $t_f$ on the generated motion can be reduced by implementing a time scaling approach similar to the one exploited in DMPs \cite{DMP}. 

\begin{figure}[!b]
\centering  
\includegraphics[width=\columnwidth]{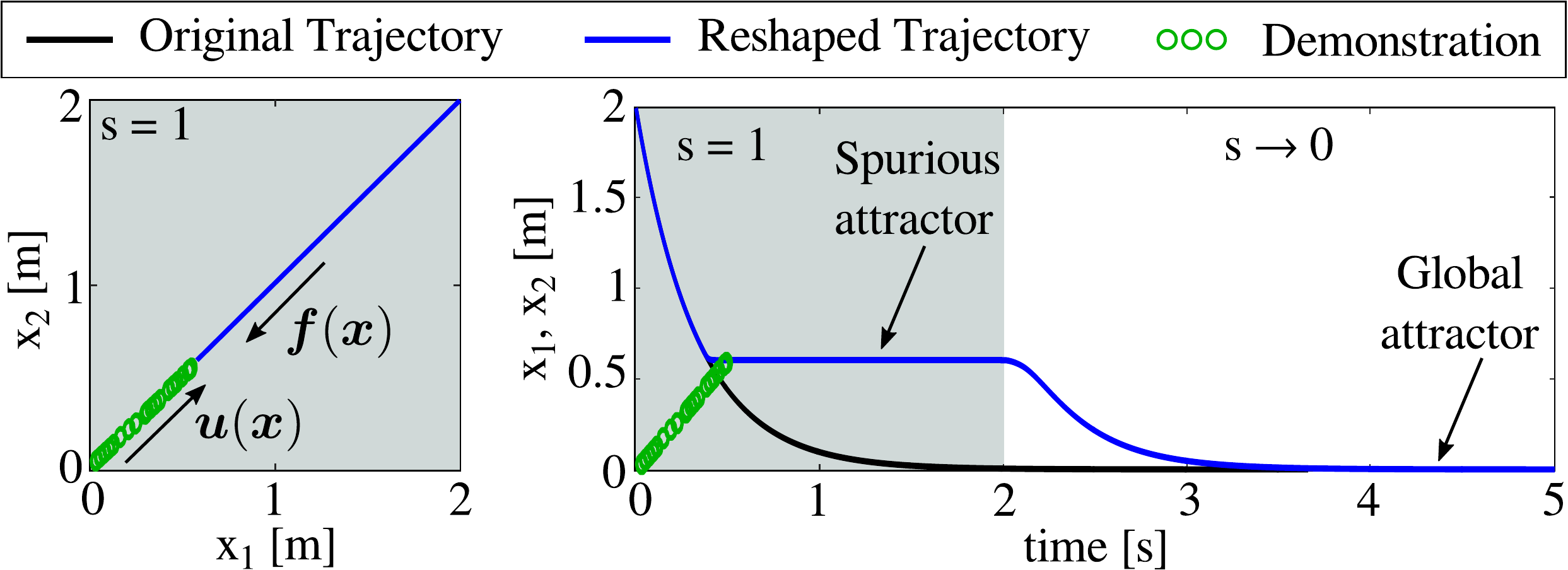}
\caption{Results obtained when the novel demonstration forces RDS to generate a spurious attractor. (Left) RDS generates a spurious attractor at $\bar{\bfx} = [0.6,\,0.6]\tr$ because $\bfu(\bar{\bfx}) = -\bff(\bar{\bfx})$. (Right) The reshaped trajectory stops into the spurious equilibrium until the control input is deactivated ($t > t_f = 2\,$s) and then converges to the global equilibrium.}
\label{fig:spurious_test}
\end{figure}

RDS has been compared with LMDS, a prominent approach in the field. The comparison has shown that trajectories generated by RDS follow the demonstrations in a more accurate manner. The higher accuracy of RDS mainly depends on the fact that an additive control input is probably more effective in imposing a different dynamics to the original system. By inspecting  (\ref{eq:train_ds}), in fact, it is clear that the learned control $\bfu(\bfx)$ cancels out the original system dynamics $\bff(\bfx)$ to impose the demonstrated dynamics $\dot{\bfx}_d$. As shown in Sec. \ref{sec:experiments}-C, RDS has the advantage to be directly applicable to high dimensional spaces, while LMDS requires the computation of a suitable rotation matrix. In spaces with more than three dimensions multiple parameterizations of a rotation are possible and all require at least $n(n-1)/2$ parameters \cite{mortari2001rigid}, while the control input in RDS is a uniquely defined $n$-dimensional vector. As a final remark, recall that LDMS  can encode periodic movements, while generating stable periodic orbits with RDS is still an open problem.

%===================== ROBOT EXPERIMENTS ===========================================%
%\input{Sections/Robot_Experiments.tex}

%===================== CONCLUSION ===========================================%
\section{Conclusions and Future Work}\label{sec:conclusion}
We presented the Reshaped Dynamical Systems (RDS), an approach useful to incrementally update a predefined skill by providing novel demonstrations. RDS is able to modify the trajectory of a dynamical system to follow demonstrated trajectories, while preserving eventual stability properties. To this end, a suitable  control input is learned from demonstrations and retrieved on-line using Gaussian process regression. The procedure is incremental, meaning that the user can add novel demonstrations until the reproduced skill is satisfactory. Experimental results show the effectiveness of the proposed approach in reshaping dynamical systems. Compared to the state-of-the-art approaches, our method has a higher reproduction accuracy and it is directly applicable to high dimensional spaces.

RDS exploits Gaussian process regression to learn and retrieve the additive control input. Gaussian processes use all the training data to regress the output, which is a drawback in incremental learning scenarios where novel demonstrations are continuously provided. The problem is alleviated in this work by using a selection algorithm that limits the number of training data. However, the applicability of other incremental learning techniques to DS reshaping has not been investigated and it will be the topic of our future research.

%The reshaping approach is effective in learning point-to-point motions. Nevertheless, DS converging towards periodic orbits (limit cycles) have been used in robotic applications to generate periodic behaviours \cite{DMP}. Compared to static equilibria, limit cycles stability has a different characterizations in terms of Lyapunov analysis. Our future research will focus on considering incremental reshaping of periodic motions while preserving their stability properties.

%===================== APPENDIX ===========================================%
%\input{Sections/AppendixLyapunov.tex}

%===================== ACKNOWLEDGEMENTS ===========================================%
%\section*{Acknowledgements}

%===================== bibliography ===========================================%
\bibliographystyle{IEEEtran}
\bibliography{bibliography}

\end{document}